\documentclass[accepted]{uai2023} 

\usepackage[american]{babel}

\usepackage[sort]{natbib} 
\usepackage{mathtools} 
\usepackage{booktabs} 
\usepackage{tikz} 

\title{A Data-Driven State Aggregation Approach\\for Dynamic Discrete Choice Models}

\author[1]{Sinong Geng}
\author[2]{Houssam Nassif\thanks{Work done while at Amazon.}}
\author[3]{Carlos A. Manzanares}
\affil[1]{%
    Computer Science Department\\
    Princeton University\\
    Princeton, NJ, USA
}
\affil[2]{%
    Meta, Seattle, WA, USA
}
\affil[3]{%
    Amazon, Seattle, WA, USA
  }

\usepackage{bbm}
\usepackage{bm}
\usepackage{mathtools}
\usepackage{amsfonts}
\newcommand{\curly}[1]{\left\{#1\right\}}
\newcommand{\norm}[1]{\left\lVert#1\right\rVert}

\usepackage{amsthm}
\usepackage{dsfont}

\newcommand{\hPi}{\hat{\Pi}}

\newcommand{\ts}{\tilde{s}}

\newcommand{\tQ}{\tilde{Q}}
\newcommand{\ttheta}{\tilde{\theta}}

\newcommand{\tM}{\tilde{M}}
\newcommand{\tL}{\tilde{L}}
\newcommand{\hL}{\hat{L}}
\newcommand{\htheta}{\hat{\theta}}

\newcommand{\tS}{\tilde{\mathcal{S}}}
\newcommand{\mS}{\mathcal{S}}
\newcommand{\mA}{\mathcal{A}}

\newcommand{\bepsilon}{\bm{\epsilon}}

\newcommand\given[1][]{\:#1\vert\:}

\newcommand{\EE}{{\mathbb{E}}}

\usepackage{mathtools}
\DeclarePairedDelimiter\abs{\lvert}{\rvert}%
\makeatletter
\let\oldabs\abs
\def\abs{\@ifstar{\oldabs}{\oldabs*}}

\usepackage{makecell}
\usepackage{comment}
\usepackage{graphicx}

\usepackage{multirow}
\usepackage{float}

\usepackage{pbox}

\usepackage{graphicx}

\usepackage{capt-of}
\usepackage{caption}
\usepackage{subcaption}
\usepackage{microtype}
\usepackage{wrapfig}

\usepackage{algorithm}
\usepackage{algorithmic}

\usepackage{comment}

\newcommand{\eq}[2]{\begin{equation} \label{eq:#1} #2 \end{equation}}
\newcommand{\eqs}[1]{\begin{equation*} #1 \end{equation*}}
\newcommand{\ali}[2]{\begin{align} \label{eq:#1} \begin{split}#2\end{split}   \end{align}}
\newcommand{\alis}[1]{\begin{align*}\begin{split} #1 \end{split}\end{align*}  }

\DeclareMathOperator*{\argmax}{arg\,max}
\DeclareMathOperator*{\argmin}{arg\,min}

\newtheorem{lemma}{Lemma}
\newtheorem{theorem}{Theorem}

\newtheorem{definition}{Definition}

\newtheorem{assumption}{Assumption}

\usepackage[inline]{enumitem}
\newlist{inlineenum}{enumerate*}{1}
\setlist*[inlineenum,1]{%
      label=(\roman*),%
}

\newsavebox{\tempboxa}
\newsavebox{\tempboxb}
\newsavebox{\tempboxc}

\newlist{todolist}{itemize}{2}
\setlist[todolist]{label=$\square$}
\usepackage{pifont}

\usepackage{scalerel}

\begin{document}
\maketitle

\begin{abstract}
    In dynamic discrete choice models, a commonly studied problem is estimating parameters of agent reward functions (also known as ``structural'' parameters) using agent behavioral data. This task is also known as inverse reinforcement learning.
    Maximum likelihood estimation for such models requires dynamic programming, which is limited by the curse of dimensionality~\citep{bellman1957markovian}. 
    In this work, we present a novel algorithm that provides a data-driven method for selecting and aggregating states, which lowers the computational and sample complexity of estimation. 
    Our method works in two stages.
    First, we estimate agent Q-functions, and leverage them alongside a clustering algorithm to select a subset of states that are most pivotal for driving changes in Q-functions. 
    Second, with these selected "aggregated" states, we conduct maximum likelihood estimation using a popular nested fixed-point algorithm~\citep{rust1987optimal}. 
   The proposed two-stage approach mitigates the curse of dimensionality by reducing the problem dimension.  
   Theoretically, we derive finite-sample bounds on the associated estimation error, which also characterize the trade-off of computational complexity, estimation error, and sample complexity.
    We demonstrate the empirical performance of the algorithm in two classic dynamic discrete choice estimation applications.    
\end{abstract}

\section{Introduction}
Dynamic discrete choice models (DDMs) are widely used to describe agent behaviours in social sciences~\citep{cirillo2011dynamic} and economics~\citep{keane2011structural}. They have attracted more recent interest in the machine learning literature~\citep{ermon2015learning,feng2020global}. 
In DDMs, agents make choices over a discrete set of actions, conditional on information contained in a set of discrete or continuous states. 
These choices generate current rewards, but they also influence future payoffs by affecting the evolution of states. 
A typical task of DDM estimation is to estimate the parameters of the hidden reward function, also known as \emph{structural parameters}~\citep{bajari2007estimating}.

Researchers have made extensive progress in identifying and estimating parameters associated with DDMs~\citep{eckstein1989specification}. That said, estimation is still challenging. 
Specifically, DDM estimation suffers from the curse of dimensionality, where the time-dependence of state evolution and action choices increases the dimensionality of possible solution paths exponentially with the  cardinality of states and actions~\citep{bellman1957markovian}. 
This curse often renders exact dynamic programming solutions infeasible for interesting and realistic empirical settings.
To mitigate this, it is natural to reduce the complexity of the state space by aggregating states together~\citep{singh1995reinforcement}.  
However, it is difficult to know, \emph{a-priori}, which states matter most. 

Existing methods for DDM estimation fall primarily into four categories.
\begin{inlineenum}
\item Classical methods in economics follow the framework of nested fixed-point maximum likelihood estimation~\citep{rust1987optimal}, which fully solves dynamic programming equations.
While such methods show great performance for problems with small state spaces, they struggle when faced with large-state spaces due to the curse of dimensionality. 
\item Alternatives include conditional choice probability methods ~\citep{aguirregabiria2002swapping,hotz1993conditional}, 
which generate computational efficiency gains by avoiding fully solving dynamic programming problems. 
They do so by exploiting inverse mappings between conditional choice probabilities and choice-specific value functions. 
That said, these methods do not, by themselves, attempt to limit the size of state spaces and often require stronger assumptions. 
\item The estimation problem can also be seen as a maximal entropy inverse reinforcement learning (IRL) problem~\citep{ermon2015learning}, which facilitates machine learning solutions~\citep{geng2020deep,yoganarasimhan2018dynamic}. 
Although machine learning methods accommodate large state spaces using powerful function approximators, these approximators require large datasets and underperform in small samples.
\item As a convenient practical technique, many researchers first aggregate states in an ad-hoc manner before applying DDM methods, in order to reduce both computational and sample complexity~\citep{rust1997using,arcidiacono2011practical,bajari2007estimating}.
An example is state discretization. 
However, state aggregation typically generates approximation errors. 
Existing state aggregation methods often choose states based on domain knowledge~\cite{dutra2011upgrades} without formally modeling approximation errors, leading to suboptimal performance.
 \end{inlineenum}

This paper proposes a data-driven method for selecting and aggregating the most relevant states associated with a widely studied class of DDMs. 
It does so in three steps. 
In step~1, we recover Q-functions using a previously proposed inverse reinforcement learning approach~\citep{geng2020deep}. 
In step~2, we identify clusters of states that generate similar Q-function values, choosing representative states from these clusters and eliminating the remaining states. 
We perform this "aggregation" by defining a distance metric based on estimated Q-function-value differences, combined with a standard clustering approach. 
In step~3, using only the selected "aggregated" states, we estimate structural parameters by employing a standard nested fixed point algorithm~\citep{rust1987optimal}. Theoretically, we derive finite-sample bounds on the associated estimation error. These bounds also characterize the  trade-off among computational complexity, estimation error, and sample complexity.
Empirically, we demonstrate the performance of our algorithm\footnote{The implementation of SAmQ is provided in \url{https://github.com/gengsinong/SAmQ}} in two well-studied dynamic discrete choice estimation applications: a bus engine replacement problem~\citep{rust1987optimal} and a simplified airline market entry problem~\citep{benkard2010simulating}.

The benefits of our approach are three-fold. 
First, by shrinking the state space to reduce computational and sample complexity, our method mitigates the curse of dimensionality faced by classical DDM estimation methods like nested fixed-point maximum likelihood estimation. If the bias of state aggregation approximation is small, the method can also lower the small-sample bias typically associated with conditional choice probability-based methods. 
Second, our final structural parameter estimation step (step~3) does not use function approximators like neural networks. 
Instead, it is based on parametric modeling of DDMs.
Compared with IRL methods, this further reduces sample complexity and provides better estimates if parametric assumptions are approximately true.  
Third, our state aggregation is data-driven in order to constrain the error caused by aggregation. 
In contrast to DDM state aggregation methods which are more ad-hoc, or which choose states based on domain knowledge or theoretical assumptions, we aggregate states by their relevance in driving estimated agent Q-functions.

Our approach is related to previously proposed techniques from different domains. 
Since it embeds the nested fixed point algorithm of ~\citep{rust1987optimal}, it is related to conditional choice probability estimation methods. 
These methods were originally developed to reduce the computational burden of nested fixed point maximum likelihood estimation~\citep{hotz1993conditional,hotz1994simulation,aguirregabiria2002swapping}. Our method is complementary to these methods, since it focuses on reducing computational complexity by limiting the state space.
Second, our method is related to IRL methods which approximately solve the dynamic programming equations for DDM estimation~\citep{geng2020deep,ermon2015learning,yoganarasimhan2018dynamic}. 
While these methods are able to handle problems with a large state spaces by powerful function approximators, they require lots of data and underperform in small samples.
Third, state aggregation in reinforcement learning and approximate dynamic programming have a long history ~\citep{singh1995reinforcement,bertsekas2018feature,huang2017large}. To the best of our knowledge, ours is the first state aggregation method applied to an IRL setting.

\section{Background}
\label{sec:background}
In Section~\ref{sec:ddm}, we specify our dynamic discrete choice model setting.
We review the popular nested fixed-point maximum likelihood estimation algorithm in Section~\ref{sec:mle}. We introduce state aggregation in Section~\ref{sec:sa}.   

\subsection{Dynamic Discrete Choice Models (DDM)} 
\label{sec:ddm}
DDM estimation can be formulated as a maximal entropy IRL problem~\citep{fu2017learning,ermon2015learning,geng2020deep}~\footnote{We detail how the IRL formulation relates the original DDM formulation of \citet{rust1987optimal} in Section 1 of Supplements.}.
Specifically, agents make decisions under a Markov Decision Problem (MDP), $M = (\mathcal{S}, \mathcal{A}, r, \gamma, P)$, with $\mathcal{S}$ denoting the state space, $\mA$ the finite action space with $n_a$ values, $r$ the reward function, $\gamma$ the discount factor, and $P$ the transition probability.
$S_t$ denotes the state variable and $A_t$ the action variable. 
For $s\in\mathcal{S}$, $a\in\mathcal{A}$, the reward function $r$ can be further defined as a function of states, actions, and parameters, i.e. $r(s,a;\theta)$, with $\theta$ denoting structural parameters of the reward function. 
The goal of DDM estimation is to estimate $\theta$ using agent decision-making behaviours. 
An accurate $\theta$ estimation is important to further counterfactual and causal analysis~\citep{fiez2022adaptive,pesaran2016counterfactual,nassif2013SAYL,dasgupta2019causal,zhang2020causal}, especially in healthcare~\citep{kuang2020ivy,geng2018temporal,geng2019parathyroid} and economics~\citep{alaluf2022reinforcement,kalouptsidi2021counterfactual}. 

Choices in empirical applications are rarely rationalized. 
It is common to assume that agents behave according to stochastic policies~\citep{ziebart2008maximum}, with $\pi(s,a)$ representing the conditional probability $\textnormal{P}(A_t = a_t \given S_t = s)$. 
Specifically, agents take stochastic energy-based policies:
$\pi(s, a) = \frac{ \exp(f(s, a))}{\sum_{a ' \in \mathcal{A}} \exp(f(s, a') da')  }$, where $f$ is usually referred to as an energy function. 
Such energy-based distributions are widely used various domains~\citep{geng2017efficient,geng2018temporal,geng2018stochastic,biswas2019seeker,geng2019partially,hinton2012practical}.
Further, agents make decisions by maximizing an entropy-augmented objective with the value function defined as:
\ali{energy-control}{
V^{\theta}(s) :=& \max_\pi \sum_{t=0}^\infty \gamma^t \, \mathds{E}[ 
 r(S_t, A_t;\theta) 
 \\&+ \mathcal{H}(\pi( S_t, \cdot))  \given S_0 = s
],
}
where $\mathcal{H}(\pi (s, \cdot)) := - \int_{\mathcal{A}} \log(\pi(s,a)) \pi(s, a) \, da $ represents information entropy. 
The superscript $\theta$ emphasizes that the value function is a function of structural parameters $\theta$ associated with the reward function.
The Q-function satisfies the following Bellman equation:
\begin{align}
\label{eq:q}
\begin{split}
Q^{\theta}&(s, a)= r(s, a;\theta) 
\\&+\max_\pi \mathds{E}{\left\{
\sum_{t=1}^\infty \gamma^t\left[
r(S_{t}, A_t;\theta) + \mathcal{H}(\pi(S_{t}, \cdot))
\right] \given s, a
\right\}}. 
\end{split}
\end{align}

In such a model, agent decision-making satisfies the following lemma (summarizing the results in \citet{geng2020deep,ermon2015learning,haarnoja2017reinforcement}). 

\begin{lemma}
\label{lem:ddm-likelihood}
Under the decision-making process described above, agents make decisions with the following choice probability
\begin{equation}
\label{eq:optimal-policy}
\textnormal{P}(A_t = a_t \given S_t = s) = \frac{ \exp(Q^{\theta}(s, a))}{\sum_{a ' \in \mathcal{A}} \exp(Q^{\theta}(s, a') )},
\end{equation}
where $Q^{\theta}(s, a)$ satisfies the following Bellman equation 
\ali{bellman}{
Q^{\theta}&(s, a) := r(s,a; \theta) 
\\&+ \gamma \EE \bigg[\log\bigg(\sum_{a'\in\mathcal{A}} \exp(Q^{\theta}(s',a')) \bigg) | s,a \bigg]. 
}
\end{lemma}

\subsection{Nested Fixed-Point Maximum Likelihood Estimation (NF-MLE)}
\label{sec:mle}
Under the setup detailed in Section~\ref{sec:ddm}, ~\citet{rust1987optimal} introduced an NF-MLE estimation algorithm widely used in economics~\citet{bajari2006semiparametric,bajari2007estimating}. We follow this framework and estimate structural parameters of the reward function by maximizing log likelihood in an iterative manner.
Specifically, consider a dataset $\mathbbm{D} = \curly{(s_i, a_i, s_i')}_{i=1}^N$ generated by the decision-making process described in Section~\ref{sec:ddm}, such that $s_i$ follows a data distribution $\mu(s)$, $a_i$ follows the optimal choice probability in \eqref{eq:optimal-policy} and $s_i'$ follows the transition. 
The partial log likelihood (abbreviated as likelihood) is derived as 
\begin{align}
\label{eq:log-likelihood}
\begin{split}
L(\mathbbm{D}; \theta):=& \frac{1}{N}\sum_{i=1}^N
\bigg(Q^{\theta}(s_i,a_i) 
\\&- \log\bigg(\sum_{a'\in\mathcal{A}} \exp(Q^{\theta}(s_i,a')) \bigg)\bigg).
\end{split}
\end{align}
Denote the true parameter as $\theta^*$.
NF-MLE maximizes \eqref{eq:log-likelihood} iteratively to estimate $\theta^*$. 
In each iteration, with a candidate $\theta$, the algorithm solves for $Q^{\theta}(s,a)$ by fixed-point iteration via the Bellman equation~\eqref{eq:bellman}. 
Then, the likelihood \eqref{eq:log-likelihood} is calculated and $\theta$ is updated accordingly.  
However, exact fixed-point iteration for $Q^{\theta}(s,a)$ is computationally costly as it requires solving a dynamic programming with high-dimensional states~\citep{bellman1957markovian}.

\subsection{State Aggregation}
\label{sec:sa}
To mitigate the issues of NF-MLE, a common practice is to choose a subset of states with a goal of making the estimation of DDMs computationally feasible~\citep{rust1987optimal,rust1997using,arcidiacono2011practical,bajari2007estimating}. We label this process \emph{state aggregation}, which can take the form of an aggregation function $\Pi(\cdot):\mathcal{S}\to \tilde{\mathcal{S}}$, 
where $\tilde{\mathcal{S}}:=\curly{\ts_1, \ts_2, \cdots \ts_{n_s}}$ represents $n_s$ aggregated states selected from the original state space $\mS$. 
In other words, $\Pi(\cdot)$ projects any state in the original state space into an aggregated state space.
With the smaller aggregated state space $\tS$, the computational burden of NF-MLE is mitigated, but it is ambiguous which states matter most \emph{a-priori}. 
It is also ambiguous as to how estimation error and state dimensionality are related, to the extent researchers are willing to trade increased estimation error for lower computational burden. 
The following section shows how DDM dimensionality and estimation error are related.

\section{Asymptotic Error and Q Error}
\label{sec:error}
We derive the asymptotic estimation error (asymptotic error for short) on structural parameters caused by state aggregation in Section~\ref{sec:bias}. 
We then separately show how state aggregation generates estimation error in Q-functions (Q error for short) in Section~\ref{sec:q-error}. 
The Q error can be used to provide an upper bound on the asymptotic error.

\subsection{Asymptotic Error of State Aggregation}
\label{sec:bias}
Aggregating states involves choosing a subset of states upon which to model DDMs. 
To the extent that DDMs are well modeled on the higher dimensional space, rather than the aggregated state space, state aggregation introduces estimation error. This error remains even with an infinite number of datasets, i.e. it is an asymptotic error.
To describe this error, we first characterize likelihood functions under state aggregation. 
Then, we rigorously define the asymptotic error of state aggregation. 

\textbf{MLE with State Aggregation $\,$}
After state aggregation, one conducts MLE on an aggregated MDP $\tilde{M} = (\tS, \mathcal{A}, \tilde{r}, \gamma, \tilde{P})$ instead of the original MDP $M = (S, \mathcal{A}, r, \gamma, P)$.
Specifically, the state space $\tS$ has a smaller cardinality than the original state space, with only $n_s$ values as demonstrated in Section~\ref{sec:sa}. 
Further, with $\xi$ as a random variable following the data distribution $\mu(\cdot)$, the reward function is redefined as 
\eqs{\tilde{r}(\ts,a;\theta):=\EE [r(\xi,a;\theta)|\Pi(\xi)=\ts].}
In words, the reward function is redefined as the average reward for the states aggregated together. 
Similarly, the transition probability is also redefined by averaging over the states aggregated together: 
\eqs{\tilde{P}(\ts',\ts, a) := \EE[ P(s'|\xi, a) \mathbbm{1}_{\Pi(s') = \ts'},\Pi(\xi) = \ts]}
where $P(s',\xi, a):=\textnormal{P}(S_{t+1} = s'|S_t =\xi, A_t = a)$ is the transition probability of the original MDP.
As a result, when conducting MLE with the aggregated MDP $\tM$, one ends up with an aggregated likelihood defined as:
\begin{align}
\label{eq:aggregated-l}
\begin{split}
    \tL(\mathbbm{D}; \theta; \Pi) :=&\frac{1}{N}\sum_{i=1}^N
\bigg(\tQ^{\theta}(\Pi(s_i),a_i) 
\\&- \log\bigg(\sum_{a'\in\mathcal{A}} \exp(\tQ^{\theta}(\Pi(s_i),a')) \bigg)\bigg),
\end{split}
\end{align}
where $\tQ^{\theta}$ denotes the Q-function of $\tilde{M}$.
We can further derive the following two characteristics of $\tQ^{\theta}$. 

\begin{lemma}
\label{lem:tq}
The Q-function of $\tilde{M}$ satisfies the following two equations:
\alis{
\tQ^{\theta}(s,a) &=\tQ^{\theta}(\Pi(s),a)
\\\tQ^{\theta}(s,a)&= \tilde{\mathcal{T}}(\tQ^{\theta}(s,a)),}
with 
\alis{
\tilde{\mathcal{T}}&(\tQ^{\theta}(s,a)):=\EE_{\xi \sim \mu(\cdot)}\bigg[r(\xi,a) 
\\&+ \gamma \EE_{s' \sim P(\cdot|\xi, a)}\bigg[\log\bigg(\sum_{a'\in\mathcal{A}} \exp(\tQ^{\theta}(\Pi(s'),a')) \bigg) | \xi,a \bigg]\\&|\Pi(\xi) = \Pi(s)\bigg].
}
Note that the internal expectation is on the next step state $s'$ while the outer expectation is on the random variable $\xi$. 
\end{lemma}
Lemma~\ref{lem:tq} follows the definition of $\tM$ and has two implications. 
First, $\tQ^{\theta}(s,a)$ returns the same value for each of the states aggregated together.
Therefore, it has only $n_s$ different values, which reduces both computational and sample complexity. 
Second, $\tQ^{\theta}(s,a)$ is a fixed point of a contraction, which allows us to estimate it by fixed-point iteration.
With the aggregated space and the projection function, we can use a $n_s\times n_a$ matrix to parameterize $\tQ^{\theta}$ and conduct fixed-point iteration to estimate $\tQ^{\theta}$.

\textbf{Asymptotic Error} Due to the discrepancy between $L$ and $\tL$, there exists an asymptotic error caused by state aggregation. 
With $\tilde{\theta}^{\Pi}:=\argmax_{\theta} \EE[\tL(\mathbbm{D};\theta, \Pi)]$, we refer to the gap between $\tilde{\theta}^{\Pi}$ and the true data generating structural parameter $\theta^*$ as the asymptotic error $\epsilon_{asy}$:
\eqs{
\epsilon_{asy}(\Pi) := \norm{\tilde{\theta}^{\Pi} - \theta^*}^2.
}
Note that the definition of $\tilde{\theta}^{\Pi}$ is asymptotic, in that it relies on knowing the expectation over $\mathbbm{D}$, i.e. having access to an infinitely sized sample.

\subsection{Q Error}
\label{sec:q-error}
It is challenging to estimate the asymptotic estimation error on $\theta^*$ directly, since $\theta^*$ is unknown.  
Instead, we focus on the Q error, which can be used to bound the asymptotic estimation error on $\theta^*$. The Q error can be estimated using IRL techniques.
\begin{definition}[Q Error] Q error is defined as
\begin{equation}
\label{eq:distance}
\epsilon_{Q}(\Pi) := \max_{(s,a)\in\mathcal{S}\times\mathcal{A}}\abs{Q^{\theta^*}(s, a) - Q^{\theta^*}(\Pi(s),a)}.
\end{equation}
\end{definition}

Multiplied by a constant related to the curvature of $\tL$,  $\epsilon_{Q}$ provides an upper bound for $\epsilon_{asy}$ (Theorem~\ref{thm:asym-main}).
This motivates us to aggregate states with an eye towards minimizing $\epsilon_{Q}$.
For any state aggregation $\Pi$, $\epsilon_{Q}$ relies on the $Q$ function, which can be estimated using maximal entropy IRL. With an estimate of $Q^{\theta^*}$, in Section~\ref{sec:method}, we show that $\epsilon_{Q}$ can be minimized by clustering states according to a distance function defined using $Q^{\theta^*}$.

\section{DDM Estimation with State Aggregation Minimizing Q Error (SAmQ)}
\label{sec:method}
Motivated by Q error, we propose a method we label SAmQ, which is an acronym for State Aggregation minimizing Q error.
The estimation procedure has three steps.
\begin{itemize}
    \item[] Step 1 Estimate $Q^{\theta^*}$ using IRL.   
    \item[] Step 2 Aggregate states by clustering.
    \item[] Step 3 Estimate structural parameters using NF-MLE with aggregated states.    
\end{itemize}

\subsection{Q Estimation by IRL}
In the first step, we use an existing IRL method to learn the Q-function $Q^{\theta^*}$. 
SAmQ works with any method that provides a good estimate to the Q-function (Assumption~\ref{asm:irl}). 
Here, we use deep PQR \citep{geng2020deep}, which estimates the Q-function in two steps:
it first estimates agent policy functions, and then it conducts fitted Q iteration. 
We summarize this step as $\hat{Q}(\cdot) \leftarrow \textnormal{DeepPQR}(\mathbbm{D})$ with $\hat{Q}(\cdot)$ denoting the estimated $Q^{\theta^*}(\cdot)$. 

\subsection{State Aggregation by Clustering}
The state aggregation minimizing Q error can be achieved by clustering on the estimated $Q^{\theta^*}$.
To see this, we consider a clustering problem with a distance function defined as
\begin{equation}
    \label{eq:metric}
d(s, s') := \max_{a \in \mathcal{A}} \abs{{Q}^{\theta^*}(s, a) - {Q}^{\theta^*}(s', a)}. 
\end{equation}
This aggregation distance~\eqref{eq:metric} describes how much states "matter" for driving changes in Q-functions.
Next, we define the projection function $\Pi (\cdot)$.  
Given a state $s\in\mS$, it returns the state $s'\in\tS$ which constitutes the \emph{center} of the cluster that $s$ belongs to.
As a result, Q error \eqref{eq:distance} is consistent with the objective function of this clustering problem. 
In other words, by clustering states with similar $Q^{\theta^*}$ values as one cluster, we can minimize the Q error. 

In practice, we use the estimated $Q^{\theta^*}$ to derive the distance function $d$ and conduct K-means clustering~\citep{hartigan1979algorithm, Kong2023NeuralInsights}.
The algorithm learns $K$ centers and clusters each observation into one of the centers.  
We allow researchers to choose the number of clusters $K$ as a hyperparameter, which is equivalent to the number of states after aggregation $n_s$.
We summarize this step as $\hat{\Pi}(\cdot)\leftarrow \textnormal{Clustering}(\mathbbm{D}, \hat{Q},n_s)$.

\subsection{NF-MLE with State Aggregation}
With the aggregation $\hPi(\cdot)$, we estimate the structural parameters of the reward function. 
Specifically, we conduct NF-MLE on aggregated states by maximizing an aggregated log-likelihood, following the algorithm described in Section~\ref{sec:bias}. 
For each iteration with a candidate $\theta$, we conduct fixed-point iteration using the sample-estimated operator $\hat{\mathcal{T}}^{\Pi}$. 
For a function $f:\mS \times \mathcal{A} \to \mathbbm{R}$, $\hat{\mathcal{T}}^{\Pi}$ is defined using the dataset $\mathbbm{D} = \curly{(s_i, a_i, s_i')}_{i=1}^N$: 
\alis{
\hat{\mathcal{T}}^{\Pi}f(\ts,a):=&\sum_{i=1,2,\cdots,N} \mathbbm{1}_{\curly{ \Pi(s_i) = \ts, a_i = a}} \bigg[r(s_i,a_i) 
\\&+ \gamma \log\bigg(\sum_{a'\in\mathcal{A}} \exp(f(\Pi(s_i'),a')) \bigg) \bigg]
\\&/\sum_{i=1,2,\cdots,N} \mathbbm{1}_{\curly{ \Pi(s_i) = \ts, a_i = a}}.
}
With the estimated Q-function denoted as $\hat{Q}^{\theta}$, the estimated aggregated likelihood is defined as
\begin{align}
\label{eq:est-aggregated-l}
\begin{split}
    \hat{L}(\mathbbm{D}; \theta; \Pi) :=&\frac{1}{N}\sum_{i=1}^N
\bigg(\hat{Q}^{\theta}(\Pi(s_i),a_i) 
\\&- \log\bigg(\sum_{a'\in\mathcal{A}} \exp(\hat{Q}^{\theta}(\Pi(s_i),a')) \bigg)\bigg).
\end{split} 
\end{align}
The procedure is summarized in Algorithm~\ref{alg:nest}.

Finally, we combine the three steps and use $\hat{\theta}$ to denote the final estimated vector of structural parameters. 
The entire SAmQ algorithm is outlined in Algorithm~\ref{alg:main}.

\begin{algorithm}[t] 
\caption{Nested Fixed-Point MLE (NF-MLE)} 
		\label{alg:nest}
		\begin{algorithmic}[1]
			\STATE {\bfseries Input:} $\mathbbm{D}$, $\Pi$
			\STATE Initialize $\theta$
			\WHILE{not converge}
			\STATE Calculate $\hat{Q}^{\theta}$ by fixed-point iteration with $\hat{\mathcal{T}}^{\Pi}$ using $\mathbbm{D}$
			\STATE Calculate the likelihood \eqref{eq:est-aggregated-l} and update $\theta$
			\ENDWHILE
			\STATE {\bfseries Return} $\theta$
		\end{algorithmic}
	\end{algorithm}

\begin{algorithm}[t]
\caption{SAmQ}
		\label{alg:main}
		\begin{algorithmic}[1]
			\STATE{\bfseries Input} Dataset: $\mathds{X}$, $n_s$.
			\STATE{\bfseries Output} $\hat{\theta}$
			\STATE $\hat{Q} \leftarrow \textnormal{DeepPQR}(\mathbbm{D})$ 
			\STATE $\hat{\Pi} \leftarrow \textnormal{Clustering}(\mathbbm{D},\hat{Q},n_s)$
			\STATE $\hat{\theta} \leftarrow \textnormal{NF-MLE}(\mathbbm{D},\hat{\Pi})$
			\STATE{\bfseries Return} $\hat{\theta}$
		\end{algorithmic}
	\end{algorithm}

\section{Theory}
In this section, we provide both asymptotic and non-asymptotic analysis for SAmQ. 
We defer the proofs to Sections 2 and 3 of the supplements. 
\subsection{Asymptotic Analysis}
We prove that the Q error can be used to bound the asymptotic estimation error of structural parameters estimated after state aggregation. 
To start with, we pose assumptions commonly used in asymptotic analysis for DDM estimation.  
\begin{assumption}
\label{asm:second-order}
For any candidate state aggregation $\Pi$, the expected aggregated likelihood function $\EE [\tL(\mathbbm{D}; \theta)]$ is strongly concave with a constant larger than $C_H > 0$.  
\end{assumption}
The intuition behind Assumption~\ref{asm:second-order} is to ensure that the aggregated log-likelihood is concave "enough." 
A concave objective function assumption is common when employing MLE-based estimators for DDMs (it is usually embedded in regularity conditions, e.g. see Proposition 2 in \citet{aguirregabiria2007sequential}).

\begin{assumption}
\label{asm:regularity}
We assume that the DDM satisfies the common regularity conditions for maximum likelihood estimation as specified in \cite{john1988maximum}. 
\end{assumption}

\begin{theorem}
\label{thm:asym-main}
Under Assumptions~\ref{asm:second-order} and~\ref{asm:regularity}, the proposed Q error provides an upper bound for the asymptotic estimation error of structural parameters, which takes the form:
\eqs{
\epsilon_{asy}(\Pi) \leq \frac{4}{C_{H}(1-\gamma)}  \epsilon_{Q}(\Pi).
}
\end{theorem}
Theorem~\ref{thm:asym-main} provides the motivation for aggregating states by minimizing the Q error, since the reward function parameter error is a function of Q error. 
Further, as we will demonstrate in Theorem~\ref{thm:finite-sample}, when the number of aggregated states $n_s$ increases, $\epsilon_Q(\Pi)$ can be very small and close to zero, making this bound especially tight. 
Note that the reward function parameter error can be constrained even more tightly if the curvature constant $C_H$ is maximized after state aggregation; but it is very challenging to ensure this methodologically.
Thus, we suggest only minimizing Q error and numerically checking $C_H$ for Assumption~\ref{asm:second-order} after aggregation.

\subsection{Non-Asymptotic Analysis}
In this section, we conduct finite-sample analysis on estimated reward function structural parameters using SAmQ. 
Specifically, we focus on the sample complexity and assume that optimization is solved without error by Assumption~\ref{asm:optimization}.  
\begin{assumption}
\label{asm:optimization}
We assume Algorithm~\ref{alg:nest} converges such that
\alis{
\hat{\mathcal{T}} \hat{Q}^{\theta}(\ts, a) = \hat{Q}^{\theta}(\ts, a), \;\textnormal{with }
\hat{\theta} = \argmax_{\theta\in\Theta} \hat{L}(\mathbbm{D}; \theta, \hat{\Pi}).
}
\end{assumption}
Further, we assume that the used IRL method and the clustering method perform well by Assumption~\ref{asm:clustering} and \ref{asm:irl}.
For several clustering and IRL methods, Assumptions~\ref{asm:clustering} and \ref{asm:irl} are proved to be satisfied with high probability \citep{geng2020deep,fu2017learning,bachem2017uniform, li2021sharper}. 
We do not repeat those analyses here. 
\begin{assumption}[Clustering Performance]
\label{asm:clustering}
Define the aggregation distance using the estimated Q-function:
\eqs{
\hat{\epsilon}_{dis}(\Pi) := \max_{(s,a)\in\mathcal{S}\times\mathcal{A}}\abs{\hat{Q}(s, a) - \hat{Q}(\Pi(s),a)}.
}
Let $\Pi^*$ be the optimal aggregation with $n_s$ aggregated states and the estimated Q-function $\hat{Q}$: 
$\Pi^* := \argmin_{\Pi\in\curly{\Pi\given|\tilde{\mathcal{S}}|=n_s }} \hat{\epsilon}_{dis}(\Pi)$.

Then, we assume that the aggregation $\hat{\Pi}$ constructed by the clustering method is close to $\Pi^*$:
\eqs{
\abs{\hat{\epsilon}_{dis}(\hat{\Pi}) -\hat{\epsilon}_{dis}(\Pi^*)}\leq \epsilon_{c}.
}
\end{assumption}
\begin{assumption}[IRL Performance]
\label{asm:irl}
We assume that  
\eqs{
\EE \bigg[\max_{(s,a)\in\mathcal{S}\times\mathcal{A}}\abs{\hat{Q}(s, a) - Q^{\theta^*}(s,a)}\bigg] \leq \epsilon_{Q}.
}
\end{assumption}

\begin{figure}[t]
\centering
\begin{minipage}[t]{0.8\linewidth}
\centering
\captionof{table}{Considered methods}
\resizebox{\textwidth}{!}{\centering
\begin{tabular}{|c|c|c|}
\hline
\textbf{Methods}&\makecell{\textbf{Category}}&\makecell{\textbf{State Aggregation} \\ \textbf{Scheme}}\\
\hline 
\textbf{SAmQ}&Proposed method&SAmQ\\\hline
NF-MLE&DDM&No aggregation\\ \hline
PQR&IRL&No aggregation\\\hline
NF-MLE-SA&DDM&By state values\\\hline  
PQR-SA&IRL&By state values\\\hline
PQR-SAmQ&IRL&SAmQ\\\hline
\end{tabular} }
\label{tab:methods}  
\end{minipage}
\end{figure}

\begin{figure*}[t]
\centering
\begin{subfigure}[b]{0.45\textwidth}
\centering
    \includegraphics[scale = 0.3]{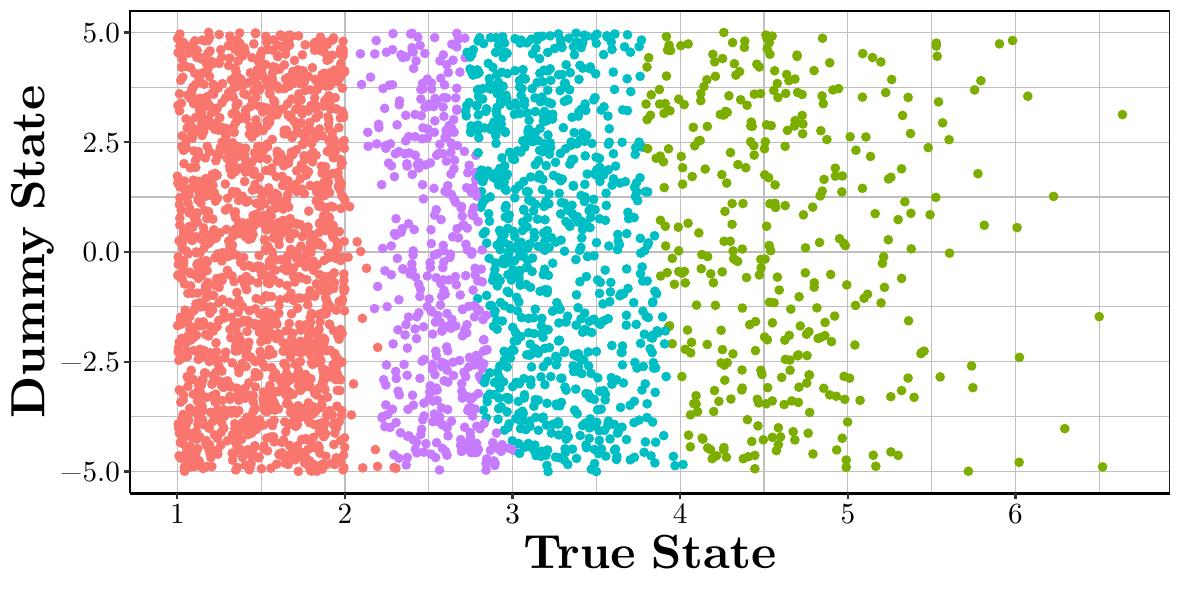}
    \caption{SAmQ}
    \label{fig:my_label}
\end{subfigure}
\begin{subfigure}[b]{0.45\textwidth}
\centering
    \includegraphics[scale = 0.3]{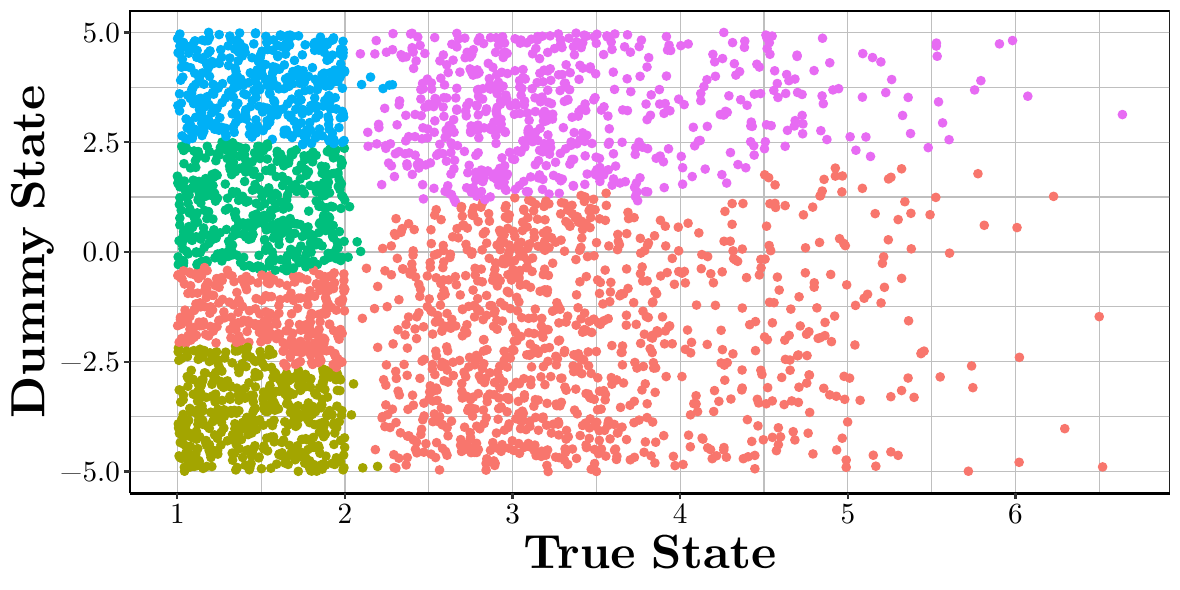}
    \caption{Ad-hoc state aggregation}
\end{subfigure}
\caption{Aggregated states for a simple example with 2-dimensional states. Each node represents a state, and each axis represents the value of one state dimension. 
\emph{The states in the same color are aggregated into one state.}
A good aggregation ignores the dummy state, and aggregates by column.} 
\label{fig:sa}
\end{figure*}
For the ease of presentation, we denote $\epsilon_P:=2\epsilon_Q+\epsilon_C$.
Importantly, although Assumption~\ref{asm:irl} requires that the used IRL method generates a good Q-function estimate, it does not imply that the IRL method will also generate a good reward function estimate. 
In fact, estimating the Q-function is easier than estimating the reward, as can be seen from \citet[Theorem 2]{geng2020deep}, where Q-function estimation has a smaller error than the reward estimation.

Next, we pose common assumptions on the data and the boundedness of the reward function. 
\begin{assumption}
\label{asm:uni}
There exists a constant $C_{uni}$ such that for a randomly picked tuple $(s_i,a_i,s_i')\in\mathbbm{D}$ and an aggregated state-action value $(\ts,a) \in \mathcal{S}\times \mathcal{A}$, 
$\textnormal{P}(\Pi(s_i) = \ts, a_i = a) \geq C_{uni}$.
\end{assumption}

\begin{assumption}
\label{asm:reward-bound}
The reward is bounded by $R_{max}$ for any $\theta \in \Theta$. 
\end{assumption}
Note that Assumption~\ref{asm:uni} assumes full data cover and can be further relaxed by advanced techniques in offline RL~\citep{rashidinejad2021bridging}. 
However, theoretical analysis for DDM estimation or IRL without full data cover is still an open question, which we defer to future work.

\begin{theorem}
\label{thm:finite-sample}
For any $\delta \in (0,1)$, let $N$ be big enough so
$
NC_{uni} -\sqrt{\frac{N\log(\frac{4n_sn_a|\Theta|}{\delta})}{2}}  \geq 1
$.
With all the assumptions aforementioned satisfied, it holds that
\alis{\textnormal{P}\bigg(\abs{\htheta - \theta^*} \leq& BiasBound 
\\&+ VarianceBound\bigg) \geq 1-\delta, \;\text{where}}
\alis{
&BiasBound:=
\frac{4}{C_H(1-\gamma)} \bigg(\frac{R_{\max}+1}{1-\gamma}\frac{4}{n_s^{\frac{1}{n_a}}-1}
+\epsilon_P\bigg),
\\ &VarianceBound:=\frac{4(R_{max}+1)}{(1-\gamma)C_H} \sqrt{\frac{\log(\frac{4|\Theta|}{\delta})}{2N}} \\&+
\frac{R_{max}+1}{(1-\gamma)^2C_H} \sqrt{ \frac{\log(\frac{8n_sn_a|\Theta|}{\delta})}{2N} } \frac{4}{C_{uni} -\sqrt{\frac{\log(\frac{4n_sn_a|\Theta|}{\delta})}{2N}}  }.
}
\end{theorem}

Theorem~\ref{thm:finite-sample} demonstrates the trade-off between bias and variance associated with state aggregation. 
\begin{itemize}
    \item $BiasBound$ corresponds to the bias caused by state aggregation, which doesn't decay with the number of samples. 
    The bias decreases as the number of aggregated states $n_s$ increases. 
\item $VarianceBound$ corresponds to the variance of DDM estimation after state aggregation, which has an order of $\frac{1}{\sqrt{N}}$ over the sample size. 
This variance part decreases as $n_s$ decreases, demonstrating the benefit of state aggregation on reducing sample complexity. 
\end{itemize}
As a result, by properly selecting $n_s$, SAmQ can improve structural parameters estimation by reducing their variance beyond the incurred bias. 
SAmQ is guaranteed to improve computational efficiency by reducing the state space, which itself is desirable in a production system~\cite{zhaoqi2022InstanceOptimal}.

\begin{figure*}[t]
\centering
\begin{minipage}[t]{\linewidth}
    \centering
    \captionof{table} {MSE for structural parameter estimation}
    {
    \begin{tabular}{cccccc} 
    \hline 
         \multirow{2}{*}{\textbf{Methods}} & \multicolumn{5}{c}{\textbf{Number of aggregated states $n_s$}}\\ \cline{2-6}
         & $5$&$10$& $50$& $100$&$1000$\\ \hline
         \textbf{SAmQ} & $0.046 \pm 0.045$ & $0.014  \pm 0.013$& $0.002 \pm 0.001$& $0.001\pm 0.000$&$0.004\pm 0.001$\\ 
         NF-MLE-SA & $5.254 \pm 2.860$ & $1.569  \pm 1.218$& $0.012 \pm 0.003$& $0.003 \pm 0.003$& $0.008 \pm 0.002$\\ 
         PQR-SA & $0.334 \pm 0.001$ & $0.355  \pm 0.019$& $0.355 \pm 0.036$& $0.332 \pm 0.003$&$0.337\pm 0.005$\\ 
         PQR-SAmQ & $1.557 \pm  0.173$& $ 0.383\pm 0.129 $& $0.354 \pm  0.018 $& $0.377 \pm 0.023$&$0.335\pm 0.004$   \\ \hline
         NF-MLE  & \multicolumn{5}{c}{$0.199 \pm 0.020$}\\ 
         PQR & \multicolumn{5}{c}{$1.276 \pm 0.094$} \\ \hline 
    \end{tabular}}
    \label{tab:bus-all}
\end{minipage}
\end{figure*}
\begin{figure*}
    \centering
    \includegraphics[scale = 0.27]{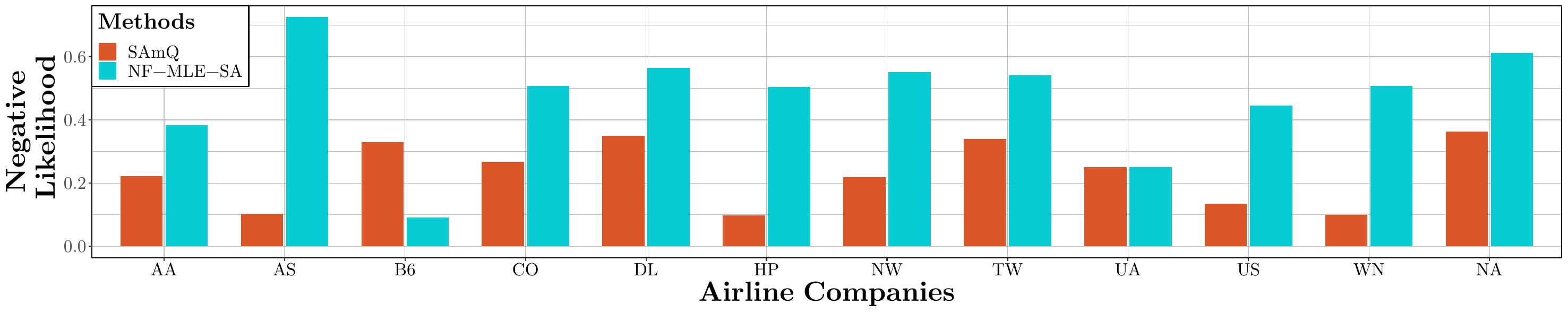}
    \caption{Prediction of airline entry behavior}
    \label{fig:airline}
\end{figure*}

\section{Experiments}
\label{sec:exp}
In this section, we demonstrate the performance of SAmQ for DDM estimation against existing methods. We use two DDM applications: a widely studied bus engine replacement problem first studied by \citet{rust1987optimal}, and a simplified airline market entry analysis~\citep{benkard2010simulating}.

\subsection{Bus Engine Replacement Analysis}
\label{sec:bus}
\textbf{Competing Methods $\,$} 
We compare SAmQ to competing reward estimation methods in both IRL and DDM with and without state aggregation.
\begin{itemize}
    \item \emph{PQR}: We first compare to deep PQR~\citep{geng2020deep} as a representative IRL method. 
    PQR estimates the policy function, Q-function and reward function, in that order. 
Since SAmQ also uses the first two steps of PQR to estimate the Q function, PQR is most related and comparable to SAmQ in the IRL category.  
\item \emph{NF-MLE}: We study NF-MLE without any state aggregation as a representative DDM estimation method~\citep{rust1987optimal}. 
\item \emph{NF-MLE-SA}: In practice, it is common to use ad-hoc state aggregation directly based on state values for NF-MLE. 
Specifically, this aggregation takes the form of state discretization, where states with similar values are aggregated together.
We label this combination as NF-MLE-SA. 
\item \emph{PQR-SA}: This method combines ad-hoc state aggregation with PQR. 
Specifically, this method first aggregates states by state values like NF-MLE-SA and then conducts PQR on the aggregated states. 
\item \emph{PQR-SAmQ}: This method first conducts state aggregation by SAmQ and then implements the full PQR algorithm on the aggregated states.

\end{itemize}

\textbf{Protocol $\,$}
We simulate the bus engine replacement problem posed by~\citet{rust1987optimal} and apply structural parameter estimation methods that seek to minimize mean square error (MSE). 
Specifically, we aim to estimate parameters of the reward function of a bus company which is faced with a task: replacing bus engines. The state variable describes utilization of a bus engine after its previous replacement. For example, this includes mileage and time. The action space has two values, representing engine replacement or regular service. With the specified reward function and simulated dynamics of states, we use soft Q iteration to solve for the optimal policy of agents and simulate decision-making data with the policy. 

\textbf{Note on Hyperparameter Tuning $\,$}
The number of aggregated states $n_s$ is a crucial parameter affecting both the accuracy and computation efficiency of SAmQ. 
In practice we can select $n_s$ using AIC or BIC for better accuracy.  
However, the selection of $n_s$ also depends on the computation burden and practical considerations. 
In some cases, users may prefer a smaller $n_s$ to solve a smaller problem that is easier to solve, even if it results in a less accurate estimation.

\textbf{Results $\,$}
We report the MSE for structural parameter estimation by each method in Table~\ref{tab:bus-all}. 
Note that SAmQ outperforms all competing methods. 
Further, compared with other DDM methods (SAmQ, NF-MLE-SA and NF-MLE), PQR-based methods (PQR-SA, PQR-SAmQ and PQR) underperform, which is consistent with our analysis on the large sample complexity of deep neural networks used by these methods. 
Comparing PQR to PQR-SAmQ and PQR-SA, we notice that state aggregation has limited improvement on PQR which is an IRL method. 
This is not surprising since PQR does not follow the MLE strategy but leverages function approximators. 
As our aggregation method is specifically designed for MLE without function approximators (see Section~\ref{sec:error}), its benefits are limited for PQR. 

\textbf{Aggregation Results $\,$}
To further examine the performance of state aggregation, we consider a simplified example with a \emph{two-dimensional} state variable. 
Among the two state components, the first state is a true state and the second state is an uninformative \emph{dummy state}, uniformly distributed in $[-5,5]$. 
The dummy state affects neither the reward nor the transition of the the first state component.  
We apply both ad-hoc state aggregation and SAmQ to derive state aggregation.
Ideally, a good state aggregation ignores the dummy state and utilizes only the true state.  
The results are reported in Figure~\ref{fig:sa}. 
We can see that SAmQ easily identifies the dummy state as the state to be discarded.

\begin{table*}[t]
\centering
\captionof{table} {MSE for structural parameter with noise to Q}
\label{tab:robust}
\begin{tabular}{llllll}
\hline
                       & $R=0.000001$             & $R=0.001$              & $R=0.01$              & $R=0.1$              & $R=0.5$              \\
                       \hline
$n_s=100$              & 0.000868              & 0.000672              & 0.001242              & 0.001627              & 0.004228              \\
$n_s=50$               & 0.001674              & 0.000884              & 0.001538              & 0.002486              & 0.008288              \\
$n_s=10$               & 0.00344               & 0.011126              & 0.002502	              & 0.021082              & 0.02605              \\ \hline
\end{tabular}
\end{table*}
\begin{table*}[t]
\centering
\captionof{table} {MSE for structural parameter with different numbers of data instances 
 $N$}
\label{tab:sample}
\begin{tabular}{cccccc}
\hline
& $N=10000$             & $N=7500$              & $N=5000$              & $N=2500$              & $N=1000$              \\
                       \hline
$n_s=100$              & 0.001035              & 0.001346              & 0.000785              & 0.000744              & 0.001306              \\
$n_s=50$               & 0.002283              & 0.002542              & 0.002198              & 0.001086              & 0.003134              \\
$n_s=10$               & 0.00513               & 0.003537              & 0.013934              & 0.011486              & 0.027419              \\ \hline
\end{tabular}
\end{table*}

\subsection{Airline Market Entry Analysis}
\textbf{Protocol}
To further demonstrate the performance of SAmQ, we study airline market entry. These entry decisions are dynamic, in that entering a market generates a fixed cost.  Specifically, airlines make decisions to enter markets defined as unidirectional city pairs. The state variables include origin/destination city characteristics, company characteristics, competitor information for each market and so on.
We apply the considered estimation methods to the data collected and pre-processed in \citet{berry2010tracing,geng2020deep,manzanares2016essays}. This setting is an adaptation of the game modeled in \citet{benkard2010simulating}.
We focus the comparison between SAmQ and NF-MLE-SA to emphasize the improvement of the proposed state aggregation scheme minimizing Q error.
Since in this application, the true reward function is unknown, we compare company behaviour prediction likelihoods on hold-out test data. 
Figure~\ref{fig:airline} reports the results, where we see that SAmQ provides better behavioural prediction compared with the competing aggregation method for most airline companies.

\subsection{Robustness to Q Estimation Error}
Note that the performance of SAmQ depends on the accuracy of Q estimation. 
However, even when there is Q estimation error, SAmQ remains relatively robust. 
To demonstrate this, under the setup of Section~\ref{sec:bus}, we added Gaussian noise with varying variances to the Q estimation in SAmQ, and report the structural parameter estimation mean squared errors (MSEs) in Table~\ref{tab:robust}.
As a measurement of the noise added, we use
$R:= \frac{Variance \, of \, Noise}{Variance \, of \, Q}$.

Importantly, our results demonstrate that SAmQ is able to provide accurate estimation even with Q function estimation errors. 
This is evident from the MSE values reported in Table~\ref{tab:robust}, which are mostly smaller than those of the competing methods in Table~\ref{tab:bus-all}. 
This robustness of SAmQ to Q estimation error can be attributed to the fact that SAmQ redoes DDM estimation after state aggregation, instead of purely relying on the estimated Q function.
Furthermore, the robustness of SAmQ to Q estimation error explains its superior performance compared to PQR in Table~\ref{tab:bus-all}. 
PQR relies entirely on the estimated Q function for reward estimation and is much more sensitive to Q estimation error than SAmQ.

\subsection{Sample Complexity}
To empirically study the sample complexity of SAmQ, we conduct additional experiments under the setting of Section~\ref{sec:bus}, where we vary the number of data instances (as a measure of sample complexity).
The results are reported in Table~\ref{tab:sample}.
First, given the same number of data instances $N$, as the number of aggregated states $n_s$ increases, the error decreases. 
Second, given the same $n_s$, as $N$ increases, the error does not always decrease. The reason is that when $N$ increases, the state aggregation needs to be more aggressive and aggregate more states together to achieve $n_s$ aggregated states. 
As a result, increasing $N$ without changing $n_s$ may hurt the estimation accuracy.

 \section{Conclusion and Future Work}
We propose a novel DDM estimation strategy with SAmQ, state aggregation minimizing Q error. 
SAmQ can significantly reduce the state space, focusing only on relevant states in a data-driven way, which brings benefits in both computational and sample complexity. 
The proposed state aggregation method is designed by minimizing the Q error caused by aggregation, and can effectively constrain the estimation error caused by state aggregation. 

 One can think of a few future directions to improve the applicability and the performance of SAmQ:
\begin{inlineenum}
    \item SAmQ currently only works for exact maximum likelihood estimation. 
    For approaches with functional approximators like many IRL techniques~\citep{geng2020deep,fu2017learning,ho2016generative}, SAmQ is not guaranteed to provide performance improvements. 
    One avenue for future work is to generalize SAmQ to such IRL methods including policy-network-based ones.
    \item The current assumptions like strong concavity and full data coverage can be relaxed by further analytical techniques in offline RL~\citep{rashidinejad2021bridging,lange2012batch,fujimoto2019off}.
    \item The methodology of SAmQ can potentially  be improved by aggregating states iteratively.
    This direction is connected to hierarchical MDP~\citep{parr1998hierarchical}.
    \item SAmQ doesn't rely on any prior domain knowledge. Finding ways to augment SAmQ with domain knowledge may help in specialized tasks or use-cases with little data~\citep{nassif2009ILPKnowledge,nassif2012LDPBN}. 
\end{inlineenum}
\bibliography{sa_uai}

\newpage
\onecolumn
\appendix
\section*{Supplements}
\section{Dynamic Discrete Choice Models in their Original Formulation}
\label{sec:ap-ddm}
In this section, we formulate dynamic discrete choice models (DDMs) using the original formulation~\citep{rust1987optimal}, and discuss its connection with the IRL formulation in Section~\ref{sec:ddm}. 
Note that the setup in this section is an alternative to the IRL formulation which our main results are based on and just is provided for completeness and comparison. 
SamQ does not require the assumptions listed in this section.

\subsection{Model} 
Agents choose actions according to a Markov decision process described by the tuple $\curly{ \curly{\mathcal{S}, \mathcal{E}}, \mathcal{A}, r, \gamma,P}$, where
\begin{itemize}
\item $\curly{\mathcal{S}, \mathcal{E}}$ denotes the space of state variables; 
\item $\mathcal{A}$ represents a set of $n_a$ actions;
\item $r$ represents an agent utility function;
\item $\gamma \in [0,1)$ is a discount factor;
\item $P$ represents the transition distribution.
\end{itemize} 

At time $t$, agents observe state  $S_t$ taking values in $\mathcal{S}$, and $\epsilon_t$ taking values in $\mathcal{E}$ to make decisions. 
While $S_t$ is observable to researchers, $\epsilon_t$ is observable to agents but not to researchers. 
The action is defined as a $n_a \times 1$ indicator vector, $A_t $, satisfying
\begin{itemize}
    \item $\sum_{j=1}^{n_a}A_{tj} = 1 $,
    \item $A_{tj}$ takes value  in $ \curly{0,1}$.
\end{itemize}
In other words, at each time point, agents make a distinct choice over $n_a$ possible actions. 
Meanwhile, $\epsilon_t$ is also a $n_a \times 1$ representing the potential shock of taking a choice.

The agent's control problem has the following value function: 
\eq{value}{
V(s,\epsilon) = \max_{\curly{a_t}_{t=0}^{\infty}} \EE \left[\sum_{t=0}^\infty \gamma^{t} r(S_t, \epsilon_t, A_t) \given s, \epsilon \right],  
}
where the expectation is taken over realizations of $\epsilon_t$, as well as transitions of $S_{t}$ and $\epsilon_{t}$ as dictated by $P$. 
The utility function $r(s_t, \epsilon_t, a_t)$ can be further decomposed into
\eqs{
    r(s_t, \epsilon_t, a_t) =u(s_t, a_t) + a_t^\top \epsilon_t,  
}
where $u$ represents the deterministic part of the utility function. Agents, but not researchers, observe $\epsilon_t$ before making a choice in each time period.

\subsection{Assumptions and Definitions}
We study DDMs under the following common assumptions. 

\begin{assumption}
\label{asm:trainsition}
 The transition from $S_t$ to $S_{t+1}$ is independent of $\epsilon_t$ 
 \eqs{
 \textnormal{P}(S_{t+1} \given S_t , \epsilon_t, A_t ) = \textnormal{P}(S_{t+1}  \given S_t, A_t).
 }
\end{assumption}

\begin{assumption}
\label{asm:epsilon}
The random shocks $\epsilon_t$ at each time point are independent and identically distributed (IID) according to a type-I extreme value distribution. 
\end{assumption}

Assumption~\ref{asm:trainsition} ensures that unobservable state variables do not influence state transitions. This assumption is common, since it drastically simplifies the task of identifying the impact of changes in observable  versus unobservable state variables. In our setting, Assumption~\ref{asm:epsilon} is convenient but not necessary, and $\bepsilon_t$ could follow other parametric distributions. As pointed out by \cite{arcidiacono2011practical}, Assumptions~\ref{asm:trainsition} and ~\ref{asm:epsilon} are nearly standard for applications of dynamic discrete choice models.
Such a formulation is proved to be equivalent to the IRL formulation in Section~\ref{sec:ddm} by \citet{geng2020deep,fu2017learning,ermon2015learning}.

\section{Proof of Theorem~\ref{thm:asym-main}}
\begin{proof}
By definition of $L$ and $\tL$, we can derive
\ali{likelihood-diff}{
L(\mathbbm{D};\theta^*) - \tL(\mathbbm{D};\theta^*)  =&\frac{1}{T} \sum_{(s,a) \in \mathbbm{D}} \bigg[Q^{\theta^*}(s,a) -\tQ^{\theta^*}(\Pi(s),a) 
\\&+ \log\bigg(\sum_{a'\in\mathcal{A}} \exp(\tQ^{\theta}(\Pi(s),a')) \bigg) - \log\bigg(\sum_{a'\in\mathcal{A}} \exp(Q^{\theta}(s,a')) \bigg) \bigg]
\\ \leq& \frac{1}{T}\sum_{(s,a) \in \mathbbm{D}} \bigg[\abs{Q^{\theta^*}(s,a) -\tQ^{\theta^*}(\Pi(s),a)} + \max_{a'\in \mathcal{A}}\abs{ Q^{\theta^*}(s,a') -\tQ^{\theta^*}(\Pi(s),a') } \bigg]
\\ \leq& 2\max_{a'\in \mathcal{A}}\abs{ Q^{\theta^*}(s,a') -\tQ^{\theta^*}(\Pi(s),a') },
}
where the first inequality is due to the fact that the log sum exp function is Lipschitz continuous with constant $1$.
Then, we take $f$ in Lemma~\ref{lem:projection-error} as $Q^{\theta^*}(s,a)$, and derive
\begin{equation}
    \label{eq:from-projection-error}
    \max_{(s,a)\in\mathcal{S}\times\mathcal{A}}\abs{Q^{\theta^*}(s, a) - \tQ^{\theta^*}(\Pi(s), a)}
    \leq \frac{2}{1-\gamma}\max_{(s,a)\in\mathcal{S}\times\mathcal{A}}\abs{Q^{\theta^*}(s, a) - Q^{\theta^*}(\Pi(s),a)}.
\end{equation}

By taking \eqref{eq:from-projection-error} to \eqref{eq:likelihood-diff}, 
\alis{
L(\mathbbm{D};\theta^*) - \tL(\mathbbm{D};\theta^*) \leq \frac{4}{1-\gamma} \max_{(s,a)\in\mathcal{S}\times\mathcal{A}}\abs{Q^{\theta^*}(s, a) - Q^{\theta^*}(\Pi(s),a)}.
}
Finally, by Lemma~\ref{lem:likelihood-bound}
\eqs{
\epsilon_{asy} \leq \frac{4}{c_{H}(1-\gamma)} \max_{(s,a)\in\mathcal{S}\times\mathcal{A}}\abs{Q^{\theta^*}(s, a) - Q^{\theta^*}(\Pi(s),a)} = \epsilon_{Q},
}
which finishes the proof. 
\end{proof}

\begin{lemma}
\label{lem:likelihood-bound}
Under Assumption~\ref{asm:second-order} and Assumption~\ref{asm:regularity},
\eqs{
\norm{\tilde{\theta} - \theta^*}^2 \leq \frac{E[L(\mathbbm{D}; \theta^*) - \tL(\mathbbm{D}; \theta^*) ]}{c_{H}}.
}
\end{lemma}
\begin{proof}

By the definition of $\ttheta$, 
\eq{key-1}{
0\leq \EE[\tL(\mathbbm{D}; \ttheta) - \tL(\mathbbm{D}; \theta^*)]  \leq\EE[ L(\mathbbm{D}; \theta^*) - \tL(\mathbbm{D}; \theta^*)].
}
Further, by Taylor expansion, we have
\eqs{
\EE[\tL(\mathbbm{D}; \ttheta) - \tL(\mathbbm{D}; \theta^*)]  = (\ttheta - \theta^*)^\top  \EE\left[-\frac{\partial ^2 \tL(\mathbbm{D}; \bar{\theta})}{\partial \theta^2}\right](\ttheta - \theta^*),
}
where $\bar{\theta} = k\theta^* + (1-k)\ttheta$ with some $k \in [0,1]$. 
Note that the first order term is zero, since $\ttheta$ maximizes $\EE[\tL(\mathbbm{D}, \theta)]$. 
By Assumption~\ref{asm:second-order}, we finish the proof. 
\eqs{
\EE[\tL(\mathbbm{D}; \ttheta) - \tL(\mathbbm{D}; \theta^*)]  = (\ttheta - \theta^*)^\top  \EE\left[-\frac{\partial ^2 \tL(\mathbbm{D}; \bar{\theta})}{\partial \theta^2}\right](\ttheta - \theta^*) \geq C_H\norm{\ttheta - \theta^*}^2.
}

\end{proof}

\begin{lemma}
\label{lem:projection-error}
For any projection function $\Pi$ defined in Section~\ref{sec:bias} and its aggregated Q function $\tQ$, the following inequality is true:
\eqs{
\max_{(s,a)\in\mathcal{S}\times\mathcal{A}}\abs{Q^{\theta^*}(s, a) - \tQ^{\theta^*}(\Pi(s), a)} \leq \frac{2}{1-\gamma}
\min_{f}\max_{(s,a)\in\mathcal{S}\times\mathcal{A}}\abs{Q^{\theta^*}(s, a) - f(\Pi(s),a)},
}
where $f(s,a):\mathcal{S}\times \mathcal{A} \to \mathbbm{R}$ is any function.  
\end{lemma}
\begin{proof}
The proof follows Theorem 3 of \cite{tsitsiklis1996feature}.  
\end{proof}

\section{Proof of Theorem~\ref{thm:finite-sample}}
\subsection{Technical Lemmas for Theorem~\ref{thm:finite-sample}}
\begin{lemma}
\label{lem:l-concentrate}
Given $\theta \in \Theta$, for any $\delta \in (0,1)$, we provide the following probabilistic bound for the estimated aggregated likelihood $\hat{L}$   
\alis{
\textnormal{P}\bigg(\abs{ \hL(\mathbbm{D};\theta) - \EE[\tL(\mathbbm{D};\theta)]} \leq& \frac{2(R_{max}+1)}{1-\gamma} \sqrt{\frac{\log(\frac{4}{\delta})}{2N}} \\&+
\frac{R_{max}+1}{1-\gamma} \sqrt{ \frac{\log(\frac{8|\tS||\mathcal{A}|}{\delta})}{2N} } \frac{2}{C_{uni} -\sqrt{\frac{\log(\frac{4|\tS||\mathcal{A}|}{\delta})}{2N}}  }   \bigg)  \geq 1-\delta,
}
where the expectation is over the sample $\mathbbm{D}$. 
\end{lemma}
\begin{proof}

By inserting $\tL(\mathbbm{D};\theta)$, we have
\begin{equation}
\label{eq:main-general}
\abs{ \hL(\mathbbm{D};\theta) - \EE[\tL(\mathbbm{D};\theta)]} \leq \abs{ \hL(\mathbbm{D};\theta) - \tL(\mathbbm{D};\theta)} + \abs{\tL(\mathbbm{D};\theta) - \EE[\tL(\mathbbm{D};\theta)]}.
\end{equation}

\paragraph{First term on the RHS of \eqref{eq:main-general}}
To start with, we consider $\abs{ \hL(\mathbbm{D};\htheta) - \tL(\mathbbm{D};\htheta)}$.
To this end, we aim to bound $\max_{(\ts,a) \in \tS \times \mathcal{A}}\abs{\tQ^{\theta}(\ts,a) - \hat{Q}^{\theta}(\ts,a)}$. 
We insert $\hat{\mathcal{T}} (\tQ^{\theta}(\ts,a))$:
\alis{
\tQ^{\theta}(\ts,a) - \hat{Q}^{\theta}(\ts,a) =& 
\tilde{\mathcal{T}}(\tQ^{\theta}(\ts,a)) - \hat{\mathcal{T}} (\tQ^{\theta}(\ts,a))+\hat{\mathcal{T}} (\tQ^{\theta}(\ts,a)) - \hat{\mathcal{T}}(\hat{Q}^{\theta}(\ts,a)). 
}
Since $\hat{\mathcal{T}}$ is a contraction with $\gamma$, we further derive
\begin{equation}
    \label{eq:q-bound-inter}
    \abs{\tQ^{\theta}(\ts,a) - \hat{Q}^{\theta}(\ts,a)} \leq \frac{\abs{\tilde{\mathcal{T}}(\tQ^{\theta}(\ts,a)) - \hat{\mathcal{T}} (\tQ^{\theta}(\ts,a))}}{1-\gamma}.
\end{equation}

By the definition of $\tilde{\mathcal{T}}$ and $\hat{\mathcal{T}}$, it can be seen that $ \hat{\mathcal{T}} (\tQ^{\theta}(\ts,a))$ is a sample average estimation to $\tilde{\mathcal{T}}(\tQ^{\theta}(\ts,a))$.
Therefore, we aim to bound the difference between the two by concentration inequalities.
Specifically, by assumption~\ref{asm:uni}  and Hoeffding's inequality, we have
\eq{p-na}{
\textnormal{P}\bigg(\sum_{i=1,2,\cdots,N} \mathbbm{1}_{\curly{ \Pi(s_i) = \ts, a_i = a}} \geq NC_{uni} - \sqrt{-\frac{1}{2}N\log(\frac{\delta}{2})}\bigg) \geq 1-\frac{\delta}{2}.
}
Further, conditional on the event $\curly{\sum_{i=1,2,\cdots,N} \mathbbm{1}_{\curly{ \Pi(s_i) = \ts, a_i = a}} \geq NC_{uni} - \sqrt{-N\log(\frac{\delta}{2})}}$, by Hoeffding's inequality and Assumption~\ref{asm:reward-bound}, for any $(\ts,a) \in \tS \times \mathcal{A}$
\ali{p-bell}{
\textnormal{P} \bigg(\bigg| \tilde{\mathcal{T}}(\tQ^{\theta}(\ts,a)) - \hat{\mathcal{T}} (\tQ^{\theta}(\ts,a))  \bigg| \leq \frac{R_{max}+1}{1-\gamma} \sqrt{ \frac{\log(\frac{4}{\delta})}{2N} } \frac{1}{C_{uni} -\sqrt{\frac{\log(\frac{2}{\delta})}{2N}}  } \bigg) \geq 1-\frac{\delta }{2}.
}

Combining \eqref{eq:p-na} and \eqref{eq:p-bell}, for a given $(\ts,a) \in \tS \times \mathcal{A}$, for any $\delta \in(0,1)$
\eqs{
\textnormal{P} \bigg(\bigg| \tilde{\mathcal{T}}(\tQ^{\theta}(\ts,a)) - \hat{\mathcal{T}} (\tQ^{\theta}(\ts,a))  \bigg| \leq \frac{R_{max}+1}{1-\gamma} \sqrt{ \frac{\log(\frac{4}{\delta})}{2N} } \frac{1}{C_{uni} -\sqrt{\frac{\log(\frac{2}{\delta})}{2N}}  } \bigg) \geq 1-\delta.
}
Next, by union bound again, we can extend the results to any $(\ts,a) \in \tS \times \mathcal{A}$
\begin{equation}
    \label{eq:p-joint}
    \textnormal{P} \bigg(\max_{\ts \in \tS, a\in\mathcal{A}}\bigg| \tilde{\mathcal{T}}(\tQ^{\theta}(\ts,a)) - \hat{\mathcal{T}} (\tQ^{\theta}(\ts,a))  \bigg| \leq \frac{R_{max}+1}{1-\gamma} \sqrt{ \frac{\log(\frac{4|\tS||\mathcal{A}|}{\delta})}{2N} } \frac{1}{C_{uni} -\sqrt{\frac{\log(\frac{2|\tS||\mathcal{A}|}{\delta})}{2N}}  } \bigg) \geq 1-\delta.
\end{equation}
Combined with \eqref{eq:q-bound-inter}, we derive:
\eqs{
\textnormal{P} \bigg(\max_{(\ts,a) \in \tS \times \mathcal{A}}\abs{\tQ^{\theta}(\ts,a) - \hat{Q}^{\theta}(\ts,a)} \leq \frac{R_{max}+1}{(1-\gamma)^2} \sqrt{ \frac{\log(\frac{4|\tS||\mathcal{A}|}{\delta})}{2N} } \frac{1}{C_{uni} -\sqrt{\frac{\log(\frac{2|\tS||\mathcal{A}|}{\delta})}{2N}}  } \bigg) \geq 1-\delta.
}
By the definition of $\tilde{L}$ in \eqref{eq:aggregated-l} and \eqref{eq:likelihood-diff}, we have
\eqs{
\textnormal{P} \bigg(\abs{\tilde{L}(\mathbbm{D}; \theta) - \hat{L}(\mathbbm{D}; \theta)} \leq \frac{R_{max}+1}{(1-\gamma)^2} \sqrt{ \frac{\log(\frac{4|\tS||\mathcal{A}|}{\delta})}{2N} } \frac{2}{C_{uni} -\sqrt{\frac{\log(\frac{2|\tS||\mathcal{A}|}{\delta})}{2N}}  } \bigg) \geq 1-\delta.
}
\paragraph{Second term on the RHS of \eqref{eq:main-general}}
Now, we consider $\abs{\tL(\mathbbm{D};\theta) - \EE[\tL(\mathbbm{D};\theta)]}$. 
By \eqref{eq:likelihood-diff} and Assumption~\ref{asm:reward-bound}, $\tL(\mathbbm{D};\htheta)$ is bounded by $\frac{2(R_{max}+1)}{1-\gamma}$.
Thus, by Hoeffding's inequality, for any $\delta \in (0,1)$
\alis{
\textnormal{P}\bigg(\abs{\EE[\tL(\mathbbm{D};\htheta)] - \tL(\mathbbm{D};\htheta)} \leq& \frac{2(R_{max}+1)}{1-\gamma} \sqrt{\frac{\log(\frac{2}{\delta})}{2N}} 
\bigg)\geq 1-\delta.
}

Therefore, by union bound, \eqref{eq:main-general} can be bounded by
\alis{
\textnormal{P}\bigg(\abs{ \hL(\mathbbm{D};\theta) - \EE[\tL(\mathbbm{D};\theta)]} \leq& \frac{2(R_{max}+1)}{1-\gamma} \sqrt{\frac{\log(\frac{4}{\delta})}{2N}} \\&+
\frac{R_{max}+1}{(1-\gamma)^2} \sqrt{ \frac{\log(\frac{8|\tS||\mathcal{A}|}{\delta})}{2N} } \frac{2}{C_{uni} -\sqrt{\frac{\log(\frac{4|\tS||\mathcal{A}|}{\delta})}{2N}}  }   \bigg)  \geq 1-\delta.
}
\end{proof}

\begin{lemma}
\label{lem:expectation}
Let $\ttheta^{\hat{\Pi}}:= \argmax_{\theta \in \Theta}\EE[\tL(\mathbbm{D}; \theta, \hat{\Pi})]$. 
Then, 
\eqs{
\norm{\theta^* - \ttheta^{\hat{\Pi}}} \leq \frac{4}{C_H(1-\gamma)} \bigg(\frac{R_{\max}+1}{1-\gamma}\frac{4}{n_s^{\frac{1}{n_a}}-1}+2\epsilon_Q + \epsilon_{c}\bigg).}
\end{lemma}

\begin{proof}
A Euclidean ball of radius $R$ in $\mathbbm{R}^{n_a}$ can be covered by $\bigg(\frac{4R+\delta}{\delta}\bigg)^{n_a}$ balls of radius $\delta$ (see Lemma 2.5 of \citet{van2000empirical}). Therefore, with $n_s$ states after aggregation, by Assumption~\ref{asm:clustering},
\eqs{
 \hat{\epsilon}(\Pi^*) \leq \frac{R_{\max}+1}{1-\gamma}\frac{4}{n_s^{\frac{1}{n_a}}-1}.
}
Further by Assumption~\ref{asm:clustering} and Assumption~\ref{asm:irl},
\eqs{
\epsilon(\hat{\Pi}) \leq \hat{\epsilon}(\Pi^*) +2\epsilon_Q + \epsilon_{c} \leq \frac{R_{\max}+1}{1-\gamma}\frac{4}{n_s^{\frac{1}{n_a}}-1} +2\epsilon_Q + \epsilon_{c}.
}
Therefore, by Theorem~\ref{thm:asym-main}
\eqs{
\norm{\theta^* - \ttheta^{\hat{\Pi}}} \leq \frac{4}{C_H(1-\gamma)} \bigg(\frac{R_{\max}+1}{1-\gamma}\frac{4}{n_s^{\frac{1}{n_a}}-1}+2\epsilon_Q + \epsilon_{c}\bigg).
}

\end{proof}
\subsection{Proof}
We first aim to bound $\EE[\tL(\mathbbm{D}; \ttheta^{\hat{\Pi}}) - \tL(\mathbbm{D}; \hat{\theta})]$, where the expectation is over $\mathbbm{D}$ only instead of $\hat{\theta}$. 
To this end, we insert $\hL(\mathbbm{D}; \ttheta^{\hat{\Pi}})$ and $\hL(\mathbbm{D}; \htheta)$:
\alis{
\EE[\tL(\mathbbm{D}; \ttheta^{\hat{\Pi}}) - \tL(\mathbbm{D}; \hat{\theta})] \leq& \EE[\tL(\mathbbm{D}; \ttheta^{\hat{\Pi}}) - \hL(\mathbbm{D};\ttheta^{\hat{\Pi}}) ] + \hL(\mathbbm{D};\ttheta^{\hat{\Pi}}) - \hL(\mathbbm{D};\htheta)+\hL(\mathbbm{D};\htheta) - \EE[\tL(\mathbbm{D};\htheta)]
\\ \leq&\abs{\EE[\tL(\mathbbm{D}; \ttheta^{\hat{\Pi}}) - \hL(\mathbbm{D};\ttheta^{\hat{\Pi}}) ]}  +\abs{ \hL(\mathbbm{D};\htheta) - \EE[\tL(\mathbbm{D};\htheta)]}. 
}
By Lemma~\ref{lem:l-concentrate} and the union bound, 
\alis{
\textnormal{P}\bigg(\max_{\theta \in \Theta}\abs{ \hL(\mathbbm{D};\theta) -\EE[\tL(\mathbbm{D};\theta)]}& \leq \frac{2(R_{max}+1)}{1-\gamma} \sqrt{\frac{\log(\frac{4|\Theta|}{\delta})}{2N}} \\&+
\frac{R_{max}+1}{(1-\gamma)^2} \sqrt{ \frac{\log(\frac{8|\tS||\mathcal{A}||\Theta|}{\delta})}{2N} } \frac{2}{C_{uni} -\sqrt{\frac{\log(\frac{4|\tS||\mathcal{A}||\Theta|}{\delta})}{2N}}  }   \bigg)  \geq 1-\delta.
}
Therefore, 
\alis{
\textnormal{P}\bigg(\EE[\tL(\mathbbm{D}; \ttheta^{\hat{\Pi}}) - \tL(\mathbbm{D}; \hat{\theta})]& \leq \frac{4(R_{max}+1)}{1-\gamma} \sqrt{\frac{\log(\frac{4|\Theta|}{\delta})}{2N}} \\&+
\frac{R_{max}+1}{(1-\gamma)^2} \sqrt{ \frac{\log(\frac{8|\tS||\mathcal{A}||\Theta|}{\delta})}{2N} } \frac{4}{C_{uni} -\sqrt{\frac{\log(\frac{4|\tS||\mathcal{A}||\Theta|}{\delta})}{2N}}  }   \bigg)  \geq 1-\delta.
}

By Assumption~\ref{asm:second-order} and a similar analysis as Lemma~\ref{lem:likelihood-bound}, 
\alis{
\textnormal{P}\bigg(\abs{\htheta - \ttheta^{\hat{\Pi}}}& \leq \frac{4(R_{max}+1)}{(1-\gamma)C_H} \sqrt{\frac{\log(\frac{4|\Theta|}{\delta})}{2N}} \\&+
\frac{R_{max}+1}{(1-\gamma)^2C_H} \sqrt{ \frac{\log(\frac{8|\tS||\mathcal{A}||\Theta|}{\delta})}{2N} } \frac{4}{C_{uni} -\sqrt{\frac{\log(\frac{4|\tS||\mathcal{A}||\Theta|}{\delta})}{2N}}  }   \bigg)  \geq 1-\delta.
}
Combined with Lemma~\ref{lem:expectation}, 
\alis{
\textnormal{P}\bigg(\abs{\htheta - \theta^*}& \leq
\frac{4}{C_H(1-\gamma)} \bigg(\frac{R_{\max}+1}{1-\gamma}\frac{4}{n_s^{\frac{1}{n_a}}-1}+2\epsilon_Q + \epsilon_{c}\bigg)+\frac{4(R_{max}+1)}{(1-\gamma)C_H} \sqrt{\frac{\log(\frac{4|\Theta|}{\delta})}{2N}} \\&+
\frac{R_{max}+1}{(1-\gamma)^2C_H} \sqrt{ \frac{\log(\frac{8n_sn_a|\Theta|}{\delta})}{2N} } \frac{4}{C_{uni} -\sqrt{\frac{\log(\frac{4n_sn_a|\Theta|}{\delta})}{2N}}  }   \bigg)  \geq 1-\delta.
}
\end{document}